\documentclass[11pt]{article}

\usepackage[letterpaper, left=1in, right=1in, top=1in,bottom=1in]{geometry}

\usepackage{parskip}
\usepackage{amsmath}
\usepackage{amsthm}

\usepackage{booktabs} 

\usepackage[utf8]{inputenc}

\usepackage{geometry}

\usepackage{graphpap,amscd,mathrsfs,graphicx,lscape,enumitem,dsfont,bm,url,subfigure}
\usepackage{epsfig,amstext,xspace}
\usepackage{algorithm,comment}
\usepackage{algorithmic}

\usepackage{paralist}

\usepackage{algorithm,algorithmic}
\usepackage[utf8]{inputenc} 
\usepackage[T1]{fontenc}    
\usepackage{hyperref}       
\usepackage{url}            
\usepackage{booktabs}       
\usepackage{amsfonts}       
\usepackage{nicefrac}       
\usepackage{microtype}      
\usepackage{graphicx} 
\usepackage{pgfplots}
\usepackage{tikz}
\usetikzlibrary{patterns}
\usepackage[font=footnotesize,labelfont=bf]{caption}

\newcommand{\CB}{\textsc{CB}}

\newcommand{\Reg}{\mathsf{Reg}}

\newcommand{\Exp}{\mathbf{E}}

\newtheorem{theorem}{Theorem}[section]
\newtheorem{remark}[theorem]{Remark}
\newtheorem{lemma}[theorem]{Lemma}
  \newtheorem{corollary}[theorem]{Corollary}

\usepackage{color}              
\usepackage[suppress]{color-edits}

\addauthor{tl}{cyan}
\addauthor{rpl}{blue}
\addauthor{vm}{green}

\begin{document}
\title{Bandits with adversarial scaling}

 \author{
 Thodoris Lykouris\thanks{Microsoft Research NYC, \texttt{thlykour@microsoft.com}. Work conducted in part while author was a Ph.D. student at Cornell University and supported by a Google Ph.D. Fellowship, and NSF grants CCF-1563714, and CCF-1408673.}
 \and  Vahab Mirrokni\thanks{Google Research NYC, \texttt{mirrokni@google.com}.}
 \and Renato Paes Leme \thanks{Google Research NYC, \texttt{renatoppl@google.com}.}
 }
\date{}
\maketitle
\begin{abstract}
We study \emph{adversarial scaling}, a multi-armed bandit model where rewards have a stochastic and an adversarial component. Our model captures display advertising where the \emph{click-through-rate} can be decomposed to a (fixed across time) arm-quality component and a  non-stochastic user-relevance component (fixed across arms). Despite the relative stochasticity of our model, we demonstrate two settings where most bandit algorithms suffer. On the positive side, we show that two algorithms, one from the  action elimination and one from the mirror descent family are adaptive enough to be robust to adversarial scaling. Our results shed light on the robustness of adaptive parameter selection in stochastic bandits, which may be of independent interest.
\end{abstract}

\section{Introduction}\label{sec:intro}
The multi-armed bandit setting is the cleanest paradigm to capture the tension between exploring information about the underlying system and exploiting the most profitable actions based on the current information. A decision-maker (or \emph{learner}) repeatedly selects among a set of $k$ actions also referred to as arms, earns the reward of the selected arm, and obtains feedback only about it. This creates a direct trade-off between learning about the performance of underexplored actions and earning the reward from those that seem the most profitable. This trade-off is prominent in applications such as display advertising where a platform needs to repeatedly select which ad to show in response to a particular pageview. The arms therefore correspond to the competing ads and, upon being shown, each ad can result in a click or not. We use this as a running example and assume that the goal of the learner is to maximize the total number of clicks (in practice, there are other goals such as revenue and user-experience, but we ignore those for simplicity).

Modeling this problem, one soon realizes that the two classical multi-armed bandit approaches fail to capture the essence of the setting. The two main approaches in the multi-armed bandit learning literature assume that the rewards obtained from each pageview are either completely adversarial or coming from identical and independent distributions (i.i.d.). However, clearly some ads have better quality and are consistently more
clickable than others, so assuming a fully adversarial model seems to ignore a lot of useful structure in the data. On the other extreme, stochastic models
assume that each ad has a fixed probability of being clicked whenever displayed (this is typically referred to as \emph{click-through-rate}). In practice, click-through-rates are known to vary due to various factors: time of the day, day of the week, seasonalities (e.g. users tend to click more before Christmas and on Black Friday) but those factors affect ads uniformly. Trac\`a and Rudin \cite{traca2015regulating} propose a model where the mean $\mu^t(a)$ reward of an arm $a$ at time $t$ is a product $$\mu^t(a) = q^t\cdot\theta(a) $$
of the intrinsic quality of the arm and a seasonality term $q^t$. In the display ads example, $q^t$ is the clickiness of the user behind her page-view and $\theta(a)$ is the intrinsic quality of the ad. The motivation of \cite{traca2015regulating} comes from retail: $\theta(a)$ is the effect certain action has in the store (the product price, which items are on sale...) and $q^t$ is the number of customers in the store in that day. In their model, the seasonality effect $q^t$ is known to the algorithm.

We depart from their model by assuming we do not have access to $q^t$. In display ads, the clickiness of users can be a function of a large number of covariates and it affects users in complex ways: certain users click more during the evening while other users click more during the day. Clickness of users in China is affected by Chinese New Year while clickiness of users in the US is affected by Thanksgiving. Estimating all those patterns is costly so it is useful to have models that exploit the structure of $\mu^t(a)$ without having to estimate the seasonality/clickness model directly.

To obtain more robust algorithms we assume $q^t$ is adversarially chosen and hence the name \emph{adversarial scaling}. If we think of $q^t$ as a seasonality effect, there is not an adversary per se in the motivation, but rather a very complex pattern which we choose not to assume anything about to get robust algorithms. However, $q^t$ can also model some actually adversarial effects. For example, an advertiser may create a botnet that does not click anything in the initial time, i.e. $q^t=0$, aiming to prolong the algorithm's exploration and hence her number of displays. This attack is very simple to run as it does not actively do anything and causes a \emph{cold start} to the algorithm as the latter needs to appropriately disregard these insignificant samples. 

As we demonstrate in this paper, despite the inherent stochasticity in these settings (best arm is better at each single round), we show that most stochastic algorithms perform poorly. This naturally generates the following question:
\begin{center}{\emph{What makes bandit algorithms robust to adversarial scaling attacks?}}
\end{center}

\paragraph{Our contribution.}
Tackling this question, we demonstrate that adaptive parameter selection is essential for robustness to adversarial scaling. In particular, we show a modification of the classical Action Elimination algorithm \cite{Even-DarManMan06} which we term \emph{AAEAS} that uses the reward of the algorithm as a proxy for the number of rounds that matter, i.e. the total scaling $\sum_tq^t$. This allows it to adapt to the intrinsic qualities $\theta(a)$ without being misled by the adversarial scaling in a way similar to the \emph{self-confident learning} technique developed for small-loss bounds in adversarial bandits. The resulting guarantees hold with high probability which is an advantage of this method. 

Aiming to understand whether this robustness is satisfied more broadly, we then focus on the \emph{Online Mirror Descent with Log-Barrier} algorithm of Foster, Li, Lykouris, Sridharan, and Tardos \cite{FosterLiLySrTa16}. Interestingly, when combined with the neat doubling trick of Wei and Luo \cite{wei2018more}, this algorithm (then termed \emph{BROAD}) seamlessly adapts to the total scaling. Its resulting guarantees hold only for the weaker pseudo-regret notion but are stronger by a factor of $k$ and the empirical performance is also enhanced.

To complement our study, we show two very simple attacks manage to make, to the best of our knowledge, all other stochastic algorithms perform ineffectively (detailed comparison in Section~\ref{sec:empirical}). Surprisingly these attacks work even against the recent breakthrough, Online Mirror Descent with Tsallis entropy \cite{ZimmertSeldin19a_best_of_both}. The latter has optimal regret guarantee for both stochastic, adversarial setting as well as regimes in between and is therefore considered as the best algorithm for such intermediate settings.  

The first attack is a purely stochastic setting where all the arms have fixed but really small means $\mu(a)$ where $\mu(a^{\star})$ is the mean of the best action. This can be viewed as adversarial scaling with $\theta(a)=\frac{\mu(a)}{\mu(a^{\star})}$ and $q^t=\mu(a^{\star})$. The bounds for these algorithms scales inversely to the absolute difference $\mu(a^{\star})-\mu(a)$. In contrast, AAEAS and BROAD have boundes based on the normalized difference $1-\frac{\mu(a)}{\mu(a^{\star})}$ which provides great improvement when the means are small as typical in click-through-rates for display advertising.

The second attack is a \emph{cold start} attack where in some initial period, the $q^t=0$; this can arise because either maliciously or organically as we discuss in the end of Section~\ref{ssec:cold_start}. The number of rounds that have passed becomes a rather irrelevant quantity in this case as what really matters is the $\sum_tq^t$. However, all algorithms other than AAEAS and BROAD use this number of rounds (many of them in ways that seem not fixable). The cold start therefore leads them to a really bad prior state and they take a very long time to recover. Surprisingly, the most effective stochastic algorithm Thompson Sampling \cite{TSAgrawal} is prone to this attack even when the cold start is really short (see Figure~\ref{fig:ts_comparison}).

\paragraph{Related work.} Our paper lies in a broader line of work that tries to achieve enhanced guarantees when data exhibit a particular benign structure. The main structure that has been utilized is the data being i.i.d. across time where many algorithms achieve logarithmic guarantees at the existence of a large gap in $\mu(a^{\star})-\mu(a)$. 
Under this structure, three previous lines of work study robustness of bandit learning to adversarial components, orthogonal to adversarial scaling. We briefly discuss them below and more elaborately review them in Section~\ref{sec:empirical}.

\begin{compactitem}
    \item Works on \emph{best of both worlds} focus on the design of algorithms that achieve improved guarantees when the input is stochastic while retaining the regret guarantees of adversarial algorithms when this is not the case. This was introduced by Bubeck and Slivkins \cite{BubeckS12} and was further studied in a sequence of works \cite{DBLP:conf/icml/SeldinS14,auer16,DBLP:conf/colt/SeldinL17,wei2018more,ZimmertSeldin19a_best_of_both,ZimmertLuoWei19}. Interestingly, our work shows that the BROAD algorithm \cite{wei2018more} is robust to adversarial scaling while Online Mirror Descent with Tsallis entropy \cite{ZimmertSeldin19a_best_of_both} that is considered superior fails at the presence of adversarial scaling. Our first algorithm can also be transformed to achieve such a best of both worlds guarantee if combined with the techniques of \cite{BubeckS12}.
    \item Works on \emph{bandits with adversarial corruptions} allow an adversary to corrupt some rounds. The adversary knows the distribution of the algorithm but \emph{not} the randomness in the arm selection (at the extreme, this can capture the adversarial bandit problem). The goal is to a) achieve the improved guarantee when there is no corruption, b) have the performance gracefully degrade with the amount of corruption, while c) being agnostic to the amount of corruption. This was initially proposed by Lykouris, Mirrokni, and Paes Leme \cite{lykouris2018stochastic} who provided regret high-probability guarantees based on a multi-layering extension of active arm elimination. The guarantees were further improved by Gupta, Koren, and Talwar \cite{gupta2019better} while further improvements were achieved for the weaker pseudo-regret notion \cite{ZimmertSeldin19a_best_of_both}. Our first algorithm can be made robust to adversarial corruptions if combined with the techniques of \cite{lykouris2018stochastic} while the second algorithm is also robust to adversarial corruptions for the notion of pseudo-regret.
    \item Last, a line of work aims to design attacks to stochastic algorithms when the adversary knows the random arm selection of the algorithm at every round and not only the arm distribution \cite{jun2018adversarial,LiuShroff19}. These works aim to minimize the amount of manipulation in the rewards needed for different objectives of the adversary and only provide attacks (instead of robust algorithms). Our algorithms obtain linear regret against such attacks (the same happens with all no-regret algorithms as these attacks give more power to the adversary than the adversarial bandit model).
\end{compactitem}

Beyond stochasticity, many works exploit other benign properties in the data to enhance adversarial guarantees. Example properties include
 small variance of the losses
\cite{Hazan:2009:BAB:1496770.1496775,wei2018more}, small effective loss range
\cite{cesa2018bandit}, small variation in the losses across rounds \cite{bubeck2019improved}, small loss of the best arm \cite{AllenbergAuGyOt06, Neu2015first, FosterLiLySrTa16, LykourisSrTa18, AllenzhuBubLi_myma},
second-order excess loss \cite{wei2018more}, locally perturbed adversarial inputs \cite{DBLP:conf/icml/ShamirS17} among others.

Finally, our model can be cast as a rank-1 assumption on the structure of the
reward function. Rank-1 assumptions have been explored in various ways in the
bandits literature, for example the rank-1 bandit model
\cite{katariya2016stochastic} and factored bandits \cite{zimmert2018factored}
assume that the action space have a cartesian product and hence the reward
in each period can be represented as a matrix over which the rank-1 assumption is
made. In other words, this rank-1 structure is on the space of actions unlike
our model which is in the space of actions cross time. Another common setting where rank-1 assumptions ara made is in
semi-parametric bandits \cite{krishnamurthy2018semiparametric} where the loss
function is a product of unknown parameters and features. A difference with
respect to our model is that the features are observed by the learner while in
our model the quality parameter $q^t$ is never observed by the learner.

\section{Model}\label{sec:model}
We study a multi-armed bandit setting with $k$ arms where each arm $a \in [k]$ is associated with an \emph{intrinsic mean} parameter
{$\theta(a)\in[0,1]$} which is
unknown to the learner. We define $\mathcal{F}^t(a)$ as the distribution of the reward of arm $a$ 
in round $t$ and 
we assume that it
{has} positive measure only on the interval $[0,1]$. The distributions are adaptively selected by the
adversary subject to the constraint that the means $\mu^t(a) = \Exp_{r \sim \mathcal{F}^t(a)}[r]$ must satisfy a rank-1 constraint, described below. Formally, the protocol between the
learner and the adversary at each round $t=1..T$ is as follows:

\begin{compactitem}
	\vspace{-\topsep}
\item The learner chooses a distribution {$p^t$} over the $k$ arms.
	\setlength{\itemsep}{0pt}
\item The adversary chooses a{n \emph{adversarial}} \emph{quality} parameter $q^t \in [0,1]$ and
  distributions $\mathcal{F}^t(a)$ supported in $[0,1]$ with mean { $\mu^t(a) =  q^t\cdot\theta(a)$.}
  	\setlength{\itemsep}{0pt}
  \item Rewards $r^t(a) \sim \mathcal{F}^t(a)$ are drawn.
  	\setlength{\itemsep}{0pt}
  \item Learner draws {$a^t \sim p^t$}
  and observes $r^t(a^t)$.
\end{compactitem}

For ease of presentation, we assume that the highest intrinsic mean is equal
to $1$, i.e. $\max_{a'} \theta(a')=1$. This can be done by appropriately
scaling down all the adversarial qualities {replacing $\theta(a)$ by
$\frac{\theta(a)}{\max_{a'} \theta(a')}$ and $q^t$ by $q^t \cdot \max_{a'} \theta(a')$.}

\paragraph{Regret.} The goal of the learner is to {maximize the aggregate reward she accumulate{s}.
To evaluate the performance of the learning algorithm, we compare this aggregate
reward with the best strategy of the learner if she had access to the reward
distributions of each arm -- in that case, the optimal strategy is to select the
arm with the highest intrinsic mean.}
{The degradation that the algorithm incurs compared to this aware setting is captured by the notion of pseudo-regret in the purely stochastic setting.} Even though we have an adversarial component on the rewards, the optimal arm to pull in each round is
always an arm that maximizes $\theta(a)$ irrespectively of the scaling chosen by the
adversarial. This allows us to define pseudo-regret in the following way:
\begin{align*} \Reg &= \max_{a \in [k]} \Exp \left[ \sum_{t=1}^T r^t(a^{\star}) - r^t(a^t) \right] =
 \sum_{t=1}^T \mu^t(a^{\star}) - \mu^t(a^t) =\sum_{t=1}^T q^t \cdot \Delta(a^t),
 \end{align*}
where $\Delta(a) = \theta(a^{\star}) - \theta(a) = 1-\theta(a)$.
In other words, the pseudo-regret is worst-case over the sequence $q(t)$ but it
is in expectation over the draws
$r^t(a) \sim \mathcal{F}^t(a)$.
\section{Algorithms robust to adversarial scaling}\label{sec:algorithm}
In this section, we show that two algorithms, one that we introduce from the active arm elimination family (Section~\ref{ssec:aae_algo}) and one existing from the mirror descent family (Section~\ref{ssec:broad}) achieve the desired robustness to adversarial scaling. A common property of both algorithms is their adaptivity on the number of rounds; they do not scale with the number of rounds which is easily targeted by adversarial scaling attacks.~\footnote{The need for such adaptivity was only known in adversarial settings prior to our work.} In the next section, we show that other algorithms with improved guarantees in the stochastic regime that do not enjoy such adaptivity have their performance severely compromised at the presence of adversarial scaling.

\subsection{Active Arm Elimination with Adversarial Scaling}\label{ssec:aae_algo}
The algorithm we introduce is based on Active Arm Elimination of Even-Dar, Mannor, and Mansour \cite{Even-DarManMan06}, but is appropriately adapted to handle the adversarial scaling of the rewards. 

Classical Active Arm Elimination keeps a set of active arms (initially all arms)
and selects arms in a round-robin fashion among the active arms, updating the
empirical mean of the selected arm as well as a confidence interval around it. The latter ensures that when samples are i.i.d. ($q^t$ same across rounds), with high probability, the actual mean lies within the confidence interval
(this comes from an application of a Chernoff bound). Once the confidence
intervals cease overlapping, we are confident that the dominated arm is not the arm with the highest intrinsic quality, thus it is safe to eliminate it
from the active arms and never select it again. 

{Trying to extend this approach to adversarial scaling, one encounters some difficulties. One issue is that classic Active Arm Elimination selects arms deterministically. The adversary can easily cause linear regret to any deterministic policy by setting $q^t = 0$
anytime the algorithm is about to pull the optimal arm. The natural way to get
around this difficulty is to pull arms in the active set randomly.}
{A more serious issue is that the number of samples (arm pulls) is no longer a
meaningful quantity.} The adversary may provide many samples initially with quality
$q^t=0$. These samples do not help in informing our estimates about where the
intrinsic qualities $\{\theta(a)\}$ lie; therefore treating these as real
samples can provide a misleading picture of the confidence intervals. Ideally,
we would like to use as \emph{effective samples} the total adversarial quality
at rounds where we selected each arm $a$, i.e. $\sum_{t:a^t=a}q^t$. Since we do
not have access to this quantity, we need to design our confidence intervals in
a way that will be robust to this adversarial scaling.

The main idea behind robustifying our confidence intervals is to use the reward
of our algorithm as a proxy of the total effective samples of each arm. At the rounds that we selected arm $a^{\star}$, the reward of the algorithm is, in expectation, equal to the effective samples of $a^{\star}$ (since $\theta(a^{\star})=1$). Since we select each non-eliminated arm with equal probability, the total reward of the algorithm $S$ serves as proxy for the effective rounds of each arm ($\frac{S}{k}$ is a lower bound and $S$ is an upper bound). Our
algorithm, 
formalized in Algorithm \ref{alg:AAEAS}, uses the proxy $S$ instead of the unknown effective samples to construct confidence intervals for the arms; in fact, there is a single confidence interval $CB(S)$.

\begin{algorithm}[H]
\caption{Active Arm Elimination with Adversarial Scaling (AAEAS)}
\label{alg:AAEAS}
\begin{algorithmic}
 \STATE Initialize the set of active arms $\mathcal{A}=[k]$, the aggregate reward for each arm $R(a)=0$, and the total reward collected by the algorithm $S=0$.
 \FOR{t=1\dots T}
  \STATE Select $a^t$ randomly across the set of active arms $\mathcal{A}$ and
  earn reward $r^t(a^t)$ \STATE Update the total reward earned: $$S\gets S+r^t(a^t)$$
  \STATE Update empirical reward of selected action: $$R(a^t)\gets R(a^t)+r^t(a^t).$$
  \STATE Eliminate arms
  based on algorithm-induced confidence intervals, i.e. remove all $a'$ from
  $\mathcal{A}$ if $$ R(a')+ \CB(S) < \max_{a\in \mathcal{A}} R(a) $$
  for confidence bound (setting $\delta'=\frac{(k+1)T}{\delta}$): $$\CB(S):= 2\sqrt{\max\big(
4 S \log(2/\delta')
  ,
  16 k \log^2(2/\delta')\big)}
 	$$
  \ENDFOR
  \end{algorithmic}
\end{algorithm}

\begin{theorem}\label{thm:scaling_result}
The AAEAS algorithm (Algorithm~\ref{alg:AAEAS}) run with $\delta=\frac{1}{T}$ has pseudo-regret  at most: 
$$
O \left( \sum_{a\neq a^{\star}} \frac{k \log\big(kT\big)}{\Delta(a)} \right)
$$
\end{theorem}
The proof follows the standard active arm elimination analysis but replaces the samples of each arm by the proxy-created confidence interval. For completeness, we provide the Chernoff bound which we use in our analysis.
\begin{lemma}[standard Chernoff bound]\label{lem:chernoff}
Let $\{X_i\}$ be independent random variables in $[0,1]$, $\mu=\Exp[\sum_i X_i]$, $\epsilon>0$:
\begin{align*}
    \Pr\big[\sum_i X_i-\mu \geq \epsilon\mu \big]\leq 2e^{-\mu\epsilon^2/3}
\end{align*}
\end{lemma}
\begin{proof}[Proof of Theorem~\ref{thm:scaling_result}]
To prove the guarantee, we need to ensure two properties: i) the arm $a^{\star}$ with highest intrinsic quality, with high probability, never gets eliminated and ii) we can bound the regret incurred by each suboptimal arm.

Our analysis is based on a few events occurring and we bound the failure probability of these events. We separate the time horizon in phases where a phase $\phi$ begins when exactly $\phi$ arms have been eliminated; denote by $\tau_{\phi}$ the round that the $\phi$-th arm gets eliminated.  First, the empirical reward $R(a)$ that we experience for each arm is close to its expected reward $\bar R(a)$, that is, for each arm $a$ and each round $t$, with probability $1-\delta'$:
\begin{align}\label{eq:reward_of_each_arm}
  R(a)\in\Big[ \bar{R}(a) \pm  \sqrt{2 \bar{R}(a) \log(2/\delta')} \Big]
 \text{ where } \quad \bar R(a)=
  \sum_{\phi=0}^{k-1}
\sum_{\tau=\tau_{\phi}+1}^{\min(t,\tau_{\phi+1})} \frac{\theta(a)\cdot q^{\tau}}{k-\phi}
\end{align}
This follows by a Chernoff bound (Lemma~\ref{lem:chernoff}), since the rewards at round $t$ are  supported in $[0,1]$ and the arms are sampled uniformly at random from the set of active arms.  

In a similar fashion, we can provide a similar confidence bound for the reward of our algorithm. For any round $t$, with probability $1-\delta'$:
\begin{equation}\label{eq:reward_of_algorithm}
  S\in\Big[ \bar S \pm \sqrt{2 \bar S \log(2/\delta')} \Big] \quad \text{for}
  \quad \bar S = \sum_a \bar R(a)
 \end{equation}
Setting all failure probabilities to $\delta'=\frac{\delta}{(k+1)T}$,
the probability that any of them fails is at most $\delta$. In the remainder of
the proof, we assume that none of these bounds fails.

\paragraph{First property.}  First, we establish that, when these bounds do not
fail, the arm $a^{\star}$ with the highest intrinsic mean does not become
eliminated. Since the mean for any arm $a\neq a^{\star}$ is at most the mean for $a^\star$, equation
\eqref{eq:reward_of_each_arm} implies:
  $$R(a) - R(a^\star) \leq 2 \sqrt{2\bar R(a^\star) \log(2/\delta')}$$
We now show that, at every round, this difference is
covered by the confidence bounds of our algorithm, i.e:
\begin{equation}\label{eq:property_1}
  \CB(S) \geq
  2 \sqrt{2\bar R(a^\star) \cdot \log(2/\delta')}
\end{equation}
  If $\bar R(a^\star) \leq 8 k \log(2/\delta') $ then \eqref{eq:property_1} holds as the second term of $\CB(S)$ is $2\sqrt{16k\log^2(2/\delta')}\geq2\sqrt{2\bar R(a^\star) \cdot \log(2/\delta')}$.
Otherwise, \eqref{eq:reward_of_algorithm} and the optimality of $a^{\star}$ imply:
  \begin{align*}S \geq \bar S - \sqrt{2\bar S \log(2/\delta')} \geq \bar R(a^\star) - 
  \sqrt{2k \bar R(a^\star) \log(2/\delta')}
\geq \frac{\bar R(a^\star)}{2}\end{align*}
since the function  $x\mapsto \frac{x}{2}-\frac{1}{2}\sqrt{x\cdot 8k\log(2/\delta')}$ is increasing for $x\geq 8k\log(2/\delta')$.

Hence, \eqref{eq:property_1} holds as the first term of $\CB(S)$ is $2\sqrt{4S\log(2/\delta')}\geq2\sqrt{2\bar R(a^\star) \cdot \log(2/\delta')}$. Since, in both cases, \eqref{eq:property_1} holds, arm $a^{\star}$ does not become eliminated.

\paragraph{Second property.}
Now we bound the regret coming from each suboptimal arm. Let's consider the
contribution to regret of arm $a\neq a^{\star}$ with gap $\Delta(a)$. First we
bound the difference of empirical rewards of two arms. By
\eqref{eq:reward_of_each_arm} and the fact that $\theta(a^\star) = 1$ we have
that:
$$R(a^\star) - R(a) \in \Big[  \Delta(a) \cdot \bar R(a^\star) 
\pm 2 \sqrt{ 2\bar R(a^\star) \log(1/\delta')} \Big] $$

We want to argue that once $\bar R(a^\star)$ is large enough, the difference between
the two arms $R(a^\star) - R(a)>\CB(S)$, which leads to arm $a$ getting eliminated and not contributing further regret. Let $T(a)$ denote the time of the elimination of arm $a$. Since arm $a$ is selected with equal probability with other active arms until then, the total expected regret from $a$ is at most:
\begin{align}\label{eq:arm_regret}
\Delta(a)\cdot\sum_{\phi=0}^{k-1}\sum_{\tau_{\phi}+1}^{\min(\tau_{\phi+1}, T(a))}\frac{q_t}{k-\phi}\leq\Delta(a)\cdot\bar{S}^{T(a)}
\end{align}
where $\bar{S}^{T(a)}$ denotes the value of $\bar S$ at round $T(a)$; recall that the latter is defined in \eqref{eq:reward_of_algorithm}. We will argue that arm $a$ is eliminated by the time we have: 
\begin{align}\label{eq:setting_elim_time}\bar S^{T(a)}>\frac{128k\log(2/\delta')}{\Delta(a)^2},\end{align}

In order for an arm to be eliminated the difference $R(a^{
\star})-R(a)$ needs to be relatively large. This difference can be expressed as:
\begin{align}
  R(a^\star) - R(a)
  &\geq \bar{R}(a^{\star})-\bar R(a) -2 \sqrt{2\big( \bar R(a^\star)\big) \log(2/\delta')}\nonumber
    \\
  &=\Delta(a) \cdot \bar R(a^\star)  - 2 \sqrt{ 2\bar R(a^\star) \log(2/\delta')}\nonumber
\\ &\geq  \Delta(a) \cdot  \frac{\bar S^{T(a)}}{k} - 2 \sqrt{ 2\frac{\bar S^{T(a)}}{k} \log(2/\delta')}.\label{eq:regret_by_alg}
\end{align}
The first inequality follows by applying \eqref{eq:reward_of_each_arm} for both arm $a$ and $a^{\star}$, and noting that $\bar R(a^{\star})\geq \bar R(a)$. The second inequality is since $\theta(a^{\star})-\theta(a)=\Delta(a)$. For the third inequality, first note that:
\begin{equation}\label{eq:lower_bound_S}\bar R(a^{\star})\geq \frac{\bar S^{T(a)}}{k},
\end{equation}
by the definitions of $\bar R(a^{\star})$ and $\bar{S}^{T(a)}$, and since $\theta(a^{\star})=1$ and $\theta(a)\leq 1$ for all $a\neq a^{\star}$. Combining \eqref{eq:lower_bound_S} and \eqref{eq:setting_elim_time}, we therefore establish that $\bar R(a^{\star})\geq \frac{\bar S^{T(a)}}{k}\geq \frac{2\log(2/\delta')}{\Delta(a)^2}$. The third inequality then follows directly since the function $ x \mapsto \Delta(a) x  - 2 \sqrt{ 2x \log(2/\delta')}$ is increasing for  $x \geq \frac{2\log(2/\delta')}{\Delta(a)^2}$.

We can re-arrange the terms in equation \eqref{eq:setting_elim_time}
and obtain:
\begin{align}
\Delta(a)\frac{\bar S^{T(a)}}{2k}\geq 2\sqrt{2\frac{\bar S^{T(a)}}{k}\log(2/\delta'})
\label{eq:subtracting}
\end{align}
Combining \eqref{eq:regret_by_alg} and \eqref{eq:subtracting} with the fact that $\Delta(a)\leq 1$, it holds:
\begin{align*}
    R(a^{\star})-R(a)\geq \frac{\bar S^{T(a)}}{2k}\Delta(a)\geq \sqrt{32\bar{S}^{T(a)} \log(2/\delta')}.
\end{align*}
The latter term is greater than the second term of the confidence bound $CB(S)$. Comparing to the first, denoting $S^{T(a)}$ the value of $S$ at $T(a)$, it holds that $S^{T(a)}\leq 2\bar S^{T(a)}$ by \eqref{eq:reward_of_algorithm} and \eqref{eq:setting_elim_time}.  Hence, the latter RHS also dominates the first term of the confidence bound $CB(S)$. As a result, if arm $a$ was not eliminated by then, it gets eliminated when the expected reward of the algorithm becomes $\bar S^{T(a)}=\frac{128k\log(2/\delta')}{\Delta(a)^2}$; by \eqref{eq:arm_regret} this implies that the expected contribution of arm $a$ to the regret is at most $\frac{128k\log(2/\delta')}{\Delta(a)}$. Summing across all suboptimal arms $a\neq a^{\star}$ and setting $\delta=1/T$ completes the proof.
\end{proof}

\begin{remark}
While the proof is written in expectation over draws $r^t(a^t) \sim F_a^t$ all arguments are high-probability arguments. The bound therefore can be
converted to a high probability regret bound (as usual in stochastic bandits)
with the difference that instead of capping the performance of each arm by
$\Delta(a) T$ as in the previous remark, we cap it by $\sqrt{T}$. Formally, we
obtain that with probability $1-\delta$ we obtain the following bound on actual
regret:
$$O\left(\sum_a
\min\left(\sqrt{T}, \frac{k\cdot \log(kT/\delta)}{\Delta(a)} \right) \right)$$

\end{remark}

\subsection{Online Mirror Descent with Log-Barrier}\label{ssec:broad}
The second algorithm that we show to be robust to adversarial scaling lies in the mirror descent family and has a stronger regularizer, log-barrier. It was initially suggested by \cite{FosterLiLySrTa16} for a fixed learning rate $\eta$, who proved that it attains first-order bounds for pseudo-regret. The update on the probabilities is the following (Algotithm 3 in \cite{FosterLiLySrTa16} adapted to rewards):
\begin{align}\label{eq:omdlb_update}
    &p^t(a^{t-1})=\frac{ p^{t-1}(a^{t-1})}{1-\eta r^t(a^{t-1})+\gamma p^{t-1}(a^{t-1})}\nonumber\\
    &p^t(a)=\frac{p^{t-1}(a)}{1+\gamma p^{t-1}(a)} \quad \forall a\neq a^{t-1}
\end{align}
where $\gamma\geq 0$ is such that $p$ is a valid probability distribution. 

Via using a neat doubling trick to update the learning rate $\eta$, Wei and Luo \cite{wei2018more} showed that, in fact, this algorithm can also attain stochastic guarantees (Algorithm 3 in \cite{wei2018more} for the particular doubling trick). In particular, the algorithm halves the learning rate and restarts once:
\begin{align*}
    \sum_{\tau=t_R}^t\sum_{a=1}^k p^\tau(a)^2 \big(\widehat{r}^\tau(a)-r^\tau(a^t)\big)^2\geq \frac{k\ln T}{3\eta^2},
\end{align*}
where $t_R$ is the time of the last restart and $\widehat{r}^\tau(a) = r^\tau(a) / p^\tau(a) \cdot {\bf 1}[a=a^t]$ is the importance sampling estimator. Their algorithm is a particular instantiation of a more general framework they termed \emph{BROAD}.

We show that the BROAD algorithm (with no modification) is robust to adversarial scaling. The proof follows from replacing the potential the potential function in the proof of Theorem 10 in \cite{wei2018more} to a potential that accommodates adversarial scaling. The remaining arguments are essentially the same as the ones used by \cite{wei2018more} in their analysis of the stochastic setting. 

\begin{theorem}\label{thm:broad_result} Online Mirror Descent with Log-Barrier with the above doubling (also known as BROAD)
has pseudo-regret at most
$
O\Big(\frac{k\log T}{\Delta}\Big)
$ where $\Delta$ is the minimum non-zero gap on intrinsic means.
\end{theorem}
We will use the following results proved in \cite{wei2018more}. The first lemma corresponds to equation (27) in their paper and the second to equation (29). They are restated here for rewards instead of losses.
\begin{lemma}[ \cite{wei2018more}]\label{lem:weiluo_aux} There is a constant $C$ such that for $p^t(a)$, $r^t(a)$ and $\hat r^t(a)$ in the BROAD algorithm it holds that:
\begin{align*}\Exp\Big[\sum_{t=1}^T r^t(a^{\star})-r^t(a)\Big] \leq Ck\ln T + C \cdot  \sqrt{(k\ln T) \Exp\Big[\sum_{t=1}^T\sum_{a=1}^k \big(p^t(a)^2\big(\widehat{r}^t(a)-r^t(a^t)\big)^2\Big]}
\end{align*}
\end{lemma}

\begin{lemma}[\cite{wei2018more}]\label{lemma:weiluo2}
Again in the context of BROAD, it holds that:
\begin{align*}
    \Exp_{a^t\in p^t}\Big[\sum_{a=1}^k p^t(a)^2 \big(\widehat{r}^t(a)-r^t(a^t)\big)^2\Big]
    \leq \max_a \Exp[r^t(a)] \cdot 2 \Exp[1-p^t(a^\star)]
\end{align*}

\end{lemma}

\begin{proof}[Proof of Theorem~\ref{thm:broad_result}] 
The pseudo-regret of the algorithm can be expressed as:
\begin{align}
    \sum_{t=1}^T\Exp\big[ r^t(a^{\star})-r^t(a^t)\big]\nonumber&=\sum_{t=1}^T\Exp\big[\sum_{a=1}^k p^t(a)\big(r^t(a^{\star})-r^t(a)\big)\big]\label{eq:broad_upper}\\
    &\geq \Exp\Big[\sum_{t=1}^T\sum_{a\neq a^{\star}}p^t(a) q^t \Delta \Big]=\Delta \Exp\Big[\sum_{t=1}^T q^t\big(1-p^t(a^{\star})\big)\Big]
\end{align}
where the inequality holds by noting that the difference in the means is at least $q^t\cdot \Delta$.
Note that in the adversarial scaling setting, we can bound the term $\max_a \Exp[r^t(a)] $ in Lemma \ref{lemma:weiluo2} by $q^t$ obtaining:
\begin{align}&\Exp_{a^t\in p^t}\Big[\sum_{a=1}^k p^t(a)^2 \big(\widehat{r}^t(a)-r^t(a^t)\big)^2\Big] \leq 2 q^t \cdot \Exp[1-p^t(a^{\star})] \label{eq:broad_lower}\end{align}
By Lemma~\ref{lem:weiluo_aux} as well as \eqref{eq:broad_lower} and \eqref{eq:broad_upper}, setting as potential function  $H=\Exp\Big[\sum_{t=1}^T{q^t}(1-p^t(a^{\star}))\Big]$ it holds that:
\begin{align*}
    H\Delta \leq \sqrt{(k\ln T) H}+k\ln T
\end{align*}
which leads to $H\leq \frac{K\ln T}{\Delta^2}$ and concludes the proof.
\end{proof}

\begin{remark}
Delving into Theorems~\ref{thm:scaling_result} and \ref{thm:broad_result}, the reader may wonder whether the logarithmic dependence on $T$ is necessary or whether it can be replaced by logarithmic dependence on $Q=\sum_t q^t$ at least for pseudo-regret guarantees. For BROAD, this is indeed the case with a simple change in the analysis: $\frac{\log(T)}{\Delta}$ only appears inside the analysis to bound the divergence term and can be replaced by $\frac{\log(Q)}{\Delta}$. For AAEAS, the logarithmic dependence on $T$ appears in the confidence bound used by the algorithm. By adapting this confidence bound, we can obtain regret at most $\frac{\log(1/p)}{\Delta}$ with probability $1-p$ leading to an expected regret of $\frac{\log(1/p)}{\Delta} + p\cdot Q$. If we have access to a known upper bound on $Q$, we can replace $p=\frac{1}{Q}$ and obtain $\frac{\log(Q)}{\Delta}$. At the absence of such an upper bound, it is not clear how to set the confidence bound in a way that achieves this goal.
\end{remark}
\section{Attacks against other stochastic algorithms}
\label{sec:empirical}

Besides AAE and BROAD which we previously discussed, there are few other algorithms available offering $\log T/\Delta$ type of guarantees for stochastic bandits. In this sections we discuss how those perform in adversarial scaling settings. We describe two adversarial scaling attacks: (a) small means; and (b) cold-start. Besides AAEAS and BROAD we show that the remaining alternatives  perform poorly in either of those cases. The algorithms we consider are:

\begin{compactitem}
    \item \emph {Upper Confidence Bound (UCB) \cite{Auer2002}} keeps track of the total reward $r(a)$ of each arm and the number of times $n(a)$ each arm was pulled. For each arm we compute the upper confidence bound of each arm as $\textsc{Ucb}(a) = \frac{r(a)}{n(a)} + \sqrt{\frac{\log t}{n(a)}}$ and deterministically pull the arm with largest $\textsc{Ucb}$. As we previously mentioned, every algorithm that deterministically selects an arm can be easily fooled by adversarial scaling by setting $q^t = 0$ when the algorithm is about to pull the optimal arm. We will see it is also tricked by much simpler (i.e. less adaptive) attacks.
    \item \emph {Thompson Sampling (TS) \cite{TSAgrawal}} is more easily described for the Bernoulli case where rewards are in $\{0,1\}$. The algorithm keeps a $\textsc{Beta}(n_0(a), n_1(a))$ prior for each arm $a$ initially set with $n_0(a) = n_1(a) = 1$. In each round, the algorithm takes a sample from each prior, chooses the arm $a$ with largest sampled value, observes the reward $r \in \{0,1\}$ and updates the prior by increasing $n_{r}(a)$ by $1$.
    \item \emph{EXP3++ \cite{DBLP:conf/icml/SeldinS14}:} While traditional EXP3 algorithms of \cite{auer1995gambling} don't offer $\log T/\Delta$ guarantees in stochastic settings, this modification does by introducing an exploration parameter tuned for each arm as a function of its past comparative performance. Each arm is explored with probability given by such parameter and with remaining probability a standard EXP3 algorithm is run. This algorithm retains the EXP3 guarantees in the adversarial regime.
    \item \emph{Tsallis Entropy \cite{ZimmertSeldin19a_best_of_both}:} A recent breatkthrough result provides an optimal algorithm (up to constants) for both stochastic and adversarial bandits via a standard mirror descent regularized by the Tsallis entropy. Remarkably, this algorithm requires no special tuning and no deviation from the standard mirror descent paradigm. The algorithm computes an unbiased estimator $\tilde{r}(a)$ of the reward of each arm and then samples an arm from the probability distribution in the solution of the following maximization problem:
    $$\max_{p \in \Delta} \sum_a \tilde{r}(a) \cdot p(a) + \frac{4}{\sqrt{t}} \left[ \sum_a \sqrt{p(a)} - \frac{1}{2}p(a) \right]$$
\end{compactitem}

In Figure \ref{fig:simple_case} we compare those algorithms in a purely stochastic instance with large means. As usually noted in the literature, the performance of Thompson Sampling is vastly superior than all other algorithms. On this instance, UCB, Tsallis and BROAD have similar perfomance, EXP3++ is somewhat worse followed by AAE and AAEAS which are notably worse. This is expected as they are the least adaptive. It is good to keep those in mind as we compare their performance on certain adversarial scaling scenarios.

\begin{figure}[h]
  \centering
  \includegraphics[scale=.5]{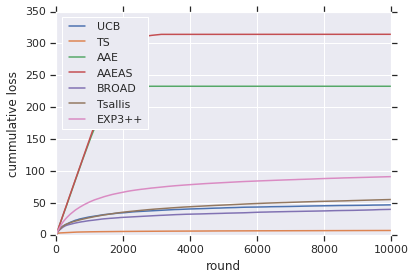}
  \caption{Comparison of different bandit algorithms for a purely stochastic instance ($q^t =1, \forall t$) with two arms with means $\mu = [0.5, 0.8]$. The cumulative in each round is the average of $100$ runs of the each algorithm.}
  \label{fig:simple_case}
\end{figure}

\subsection{Small means}\label{ssec:small_means}

Even in the absence of adversarial scaling the AAEAS and BROAD algorithm start outperforming other algorithms as the means become
smaller and smaller.  Consider a purely
stochastic instance (i.e. $q^t = 1$ for all $t$)
with only two arms with means $1 \geq \theta_1 > \theta_2 \geq 0$ and let $\Delta = \theta_1 - \theta_2$. The traditional pseudo-regret bound obtained by stochastic bandits algorihtms is $O(\log(T) / \Delta)$. The bound obtained by AAEAS and BROAD on the other hand is $O(\theta_1
\cdot \log(T) / \Delta)$, {which follows by viewing this}
purely stochastic problem as an instance with adversarial scaling with $q^t =
\theta_1$ and two arms with $\theta'_1 = 1$ and $\theta'_2 = \frac{\theta_2}{\theta_1}$. In Figure \ref{fig:small means} we compare the same algorithms in an instance with two
Bernoulli arms having means $0.005$ and $0.001$. We see that while the performance of AAEAS and BROAD is unaffected by scaling the means down, the performance of UCB, AAE, Tsallis and EXP3++ degrades despite the fact that the relative strength of both arms remains almost the same. 

\begin{figure}[h]
  \centering
  \includegraphics[scale=.5]{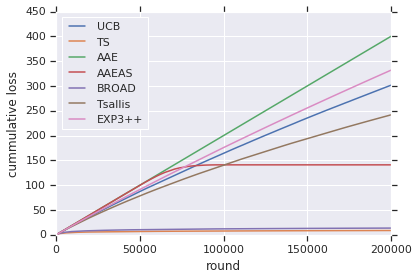}
  \caption{A purely stochastic instance ($q^t =1, \forall t$) with two arms with small means $\mu = [0.005,0.001]$. The cumulative in each round is the average of $100$ runs of the each algorithm. }
  \label{fig:small means}
\end{figure}

A direct consequence of Theorem \ref{thm:scaling_result} and Theorem \ref{thm:broad_result} is the following (we treat the number of arms $k$ as a constant in order to unify the statements for AAEAD and BROAD):

\begin{corollary}
  Given a stochastic bandit instance where rewards are supported in $[0,1]$  with means $\theta(a)$, then the performance of AAEAS and BROAD:
  $O\left( \max_a \theta_a \cdot \sum_a \frac{\log(T)}{\Delta(a)} \right)$.
\end{corollary}

Interestingly, the performance of Thompson Sampling seems unaffected by adversarial scaling. It is not clear to us how to generalize the proof of \cite{TSAgrawal} to explain the good performance of Thompson Sampling in the small mean regimes. We leave understanding this phenomenon as an open problem.

\subsection{Cold start attack}\label{ssec:cold_start}

\begin{figure}[h]
  \centering
  \includegraphics[scale=.5]{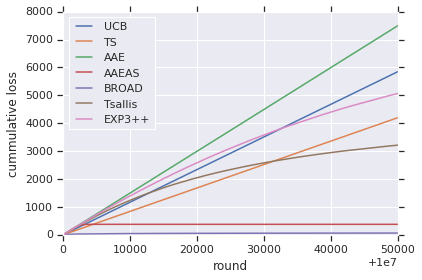}
  \caption{Comparison of different bandit algorithms on an extreme cold start instance: $q^t = 0$ for $t < t_0 = 10^7$ and $q^t = 1$ afterwards. Arm means are $\mu = [0.5, 0.8]$. We only plot the rounds after $t_0$ since the loss up to that point is zero. }
  \label{fig:cold_start}
\end{figure}

A very effective type of attack against randomized
algorithms (even Thompson Sampling) is the \emph{cold start} attack, where for the first $t_0$ periods,
the adversary chooses $q^t = 0$ giving the algorithm the impression it has pulled a lot of arms although there are effectively no pulls. In Figure \ref{fig:cold_start} we exhibit an extreme form of this attack, where we have a very long cold-start period ($t_0 = 10^7$) and from then on we have a standard stochastic instance with means $0.5$ and $0.8$. The performance of all algorithms except AAEAS and BROAD is severely hurt. In fact this performance degradation can become arbitrarly bad as $t_0 \rightarrow \infty$. This happens for different reasons depending on the type of algorithm:

\begin{compactitem}
    \item For confidence-bound based algorithms like UCB and AAE, the cold start attack produces the impression that the arms have much smaller mean than they actually have. Since the confidence bounds scale with the inverse square root of the empirical gap, the exploration phase can be arbitrarly extended as $t_0$ grows to infinity.
    \item Tsallis and EXP3++ are mirror-descent based algorithms that have a learning rate schedule that depends directly on the number of rounds ($\eta_t = 1/\sqrt{t}$ in either case). For those the cold-start will cause the learning rate to start at $1/\sqrt{t_0}$ which is much smaller than the learning rate that would be required for a stochastic instance with larger means. Note that while BROAD is also based on mirror descent, the learning rate is adaptively tuned based on the rewards and it not directly depending on the numbers of rounds. This ability of adaptive tuning also enables first-order bounds for BROAD \cite{FosterLiLySrTa16} and stems from the strong log-barrier regularizer that effectively deals with the variance in the second-order term. In contrast, Tsallis and EXP3++ have weaker regularizers and do not admit first-order bounds; for the same reason, we believe that it is unlikely that their learning rate can be appropriately tuned to circumvent this issue. 
    \item Thompson Sampling is based on keeping a prior on the means of the arms. A large number of cold-start periods leads the algorithm to a state with very skewed priors. Once we reach the end of the cold-start period at $t_0$ each arm will have a beta distribution around $\textsc{Beta}(t_0/k,1)$ instead of  $\textsc{Beta}(1,1)$.
\end{compactitem}

\begin{figure}[h]
  \centering
  \includegraphics[scale=.5]{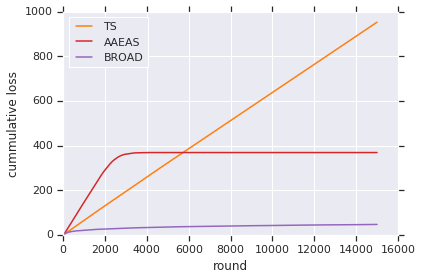}
  \caption{Comparing AAEAS with Thompson Sampling on a cold start instance ($t_0
  = 25 $) with means $\mu=[0.5, 0.8]$. The $y$-axis is the cummulative loss of the algorithm averaged over $100$ runs.}
  \label{fig:ts_comparison}
\end{figure}

In Figure \ref{fig:ts_comparison} we compare Thompson Sampling and AAEAS in an instance with a very small number of cold start rounds ($t_0 = 25$) and then run the algorithm for another $T=30000$ rounds. 
Interestingly, we see that, even when the cold start is really small, the effect of the attack is long-lasting, which provides powerful evidence that the
phenomenon described is, in fact, an actual concern when deploying Thompson Sampling. We also note that this phenomenon does not even need to be caused due
to an adversarial source. It could well occur that the initial samples are less
effective because, e.g., initially the advertised product is not yet well
established which leads customers to prefer alternative options. In fact, recent
work has suggested that, under competition, learning algorithms generally suffer from such an effect due to the exploration they need to perform in
the beginning; see \cite{MansourSlWu18innovation} and \cite{AridorLiSlWu19} and for a
relevant discussion.

\subsection{Discussion on algorithms not run}
Finally, we would like to mention that there are a few other algorithms with logarithmic stochastic guarantees that we decided to not include. The original best of both worlds algorithm SAO \cite{BubeckS12} and its follow-up SAPO \cite{auer16} are rooted in AAE and switch to EXP3 if some test fails. Similarly, multi-layer AAE \cite{lykouris2018stochastic} provides a way to robsutify AAE to adversarial corruptions. Since the instances are mostly stochastic and therefore the tests are not expected to fail, their performance is strictly inferior to the one of AAE. Another line of work focuses on stochastic algorithms from the EXP3 family that can effectively select non-stationarities in the environment, e.g. the R.EXP3 algorithm \cite{Besbes2014}.  These algorithms again come with horizon-dependent learning rates and are expected to suffer similarly with Tsallis and EXP3++ (that belong to the same family). Finally, for the same reason, we did not run the recent HYBRID algorithm \cite{ZimmertLuoWei19} which extends upon Tsallis but again uses horizon-dependent learning rates.
\section{Conclusion}\label{sec:conclusions}
In this work, we suggest a new intermediary model between stochastic and adversarial bandits where an adversary can rescale all rewards in a given round by the same factor. We show that two adaptive algorithms are robust to this adversarial scaling and provide two natural attacks that demonstrate that other stochastic algorithms are not. 

There are two nice open questions coming from our work. AAEAS is weaker by a factor of $k$ compared to BROAD; we believe that our analysis is tight in that matter but it would be interesting to see if some alternative modification of AAE can remove this dependence. Moreover, Thompson Sampling, although ineffective for the \emph{cold start} attack, has very good performance in the \emph{small means} attack; this suggests that its analysis could become more tight to scale with the ratio of the means (rather than their difference).

\bibliographystyle{alpha}
\bibliography{bandits}

\end{document}